\documentclass[a4paper,11pt]{article}
\usepackage{a4wide}
\usepackage[british]{babel}
\usepackage{enumerate}
\usepackage{graphicx}
\usepackage{amsmath,amsfonts,amssymb,amsthm}
\usepackage{calc}
\usepackage[colorlinks,citecolor=blue]{hyperref}
\usepackage{bm}
\usepackage{natbib}

\newtheorem{theorem}{\bf Theorem}[section]
\newtheorem{lemma}[theorem]{\bf Lemma}
\newtheorem{proposition}[theorem]{\bf Proposition}

\newtheorem{assumption}[theorem]{\bf Assumption}

\newtheorem{definition}[theorem]{\bf Definition}

\def \bE {\mathbb{E}}

\def \bN {\mathbb{N}}

\def \bR {\mathbb{R}}

\def \cF {{\cal F}}

\def \cH {{\cal H}}

\def \cL {{\cal L}}
\def \cM {{\cal M}}
\def \cN {{\cal N}}
\def \cO {{\cal O}}
\def \cP {{\cal P}}

\def \cS {{\cal S}}

\def \cW {{\cal W}}

\def \Lip {\,{\rm Lip}\,}
\def \MMD {\,{\rm MMD}\,}

\makeatletter
\def\blfootnote{\xdef\@thefnmark{}\@footnotetext}
\makeatother

\begin{document}

\title{On the capacity of deep generative networks for approximating distributions}

%
%
%

\author{
	Yunfei Yang \thanks{Department of Mathematics, The Hong Kong University of Science and Technology, Clear Water Bay, Kowloon, Hong Kong} \footnotemark[2] 
	\and
	Zhen Li \thanks{Theory Lab, Huawei Technologies Co., Ltd., Shenzhen, China}
	\and
	Yang Wang \footnotemark[1]
}
\date{}\maketitle
\blfootnote{Correspondence to: Yunfei Yang (yyangdc@connect.ust.hk).}

\begin{abstract}
We study the efficacy and efficiency of deep generative networks for approximating probability distributions. We prove that neural networks can transform a low-dimensional source distribution to a distribution that is arbitrarily close to a high-dimensional target distribution, when the closeness are measured by Wasserstein distances and maximum mean discrepancy. Upper bounds of the approximation error are obtained in terms of the width and depth of neural network. Furthermore, it is shown that the approximation error in Wasserstein distance grows at most linearly on the ambient dimension and that the approximation order only depends on the intrinsic dimension of the target distribution. On the contrary, when $f$-divergences are used as metrics of distributions, the approximation property is different. We show that in order to approximate the target distribution in $f$-divergences, the dimension of the source distribution cannot be smaller than the intrinsic dimension of the target distribution. 

\end{abstract}

\section{Introduction}

In recent years, deep generative models have made remarkable success in many applications such as image synthesis \citep{reed2016generative}, style transfer \citep{gatys2016image}, medical imaging \citep{yi2019generative} and natural language generation \citep{bowman2016generating}. In these applications, the probability distributions of interest are often high-dimensional, highly complex and computationally intractable. Typical deep generative models, such as variational autoencoders (VAEs) and generative adversarial networks (GANs), use deep neural networks to generate these complex distributions from simple and low-dimensional source distributions \citep{kingma2014auto,goodfellow2014generative,arjovsky2017wasserstein}. Despite their success in practice, the theoretical understanding of these models is still very limited. One of the fundamental questions on deep generative models is their capacity of expressing probability distributions. Specifically, is it possible to approximate a high-dimensional distribution by transforming a low-dimensional distribution using neural networks? What metrics of distributions should be used? And what is the required size of the network for a given accuracy? 

\subsection{Our contributions}
In this paper, we study the expressive power of neural networks for generating distributions and provide some answers to the above questions. Specifically, for a low-dimensional probability distribution $\nu$ on $\bR^m$, we consider how well a high-dimensional probability distribution $\mu$ defined on $\bR^d$ can be approximated by the push-forward distribution $\phi_\# \nu$, using the ReLU neural network $\phi: \bR^m\to \bR^d$ as a transportation map. To quantify the approximation error, we consider three typical types of metrics (discrepancies) used in generative models: Wasserstein distances, maximum mean discrepancy (MMD) and $f$-divergences.

For Wasserstein distances and MMD, we construct a neural network $\phi$ such that the generated distribution $\phi_\# \nu$ is arbitrarily close to the target distribution $\mu$. Approximation error bounds are also obtained in terms of the width and depth of neural network. Comparing with existing works \citep{lee2017ability,bailey2018size,perekrestenko2020constructive,lu2020universal} on similar topics, our results only make very weak assumptions on the distributions: absolute continuity for source distribution $\nu$ and moment conditions for the target $\mu$. And we also allow the dimension of the source distribution to be different from the dimension of the target distribution. Furthermore, it is proved that the approximation orders in Wasserstein distances only depend on the intrinsic dimension of the target distribution, which indicates that generative networks can overcome the curse of dimensionality. To the best of our knowledge, this is the first result in distribution approximation by generative networks that shows the dependency on the intrinsic dimensionality of the target distribution. 

For $f$-divergences, a similar argument as in \citep{arjovsky2017towards} shows that, when the dimension of the source distribution is smaller than the intrinsic dimension of the target distribution, the $f$-divergences of the target and the generated distributions are positive constants. Hence, in this case, it is impossible to approximate the target distribution using neural networks. Our results suggest that, from an approximation point of view, $f$-divergences are less adequate as metrics of distributions than Wasserstein distances and MMD for deep generative models.

\subsection{Related work}

The expressive power of neural networks for approximating functions has been studied extensively in the past three decades. Early works \citep{cybenko1989approximation,hornik1991approximation,barron1993universal,pinkus1999approximation} showed that two-layer neural networks are universal in the sense that they can approximate any continuous functions on compact sets, provided that the width is sufficiently large. In particular, Barron \citep{barron1993universal} gave an upper bound of the approximation error when the function of interest satisfies certain conditions on Fourier's frequency domain. Recently, the capacity of deep neural networks for approximating certain classes of smooth functions has been quantified in terms of number of parameters \citep{yarotsky2017error,yarotsky2018optimal,petersen2018optimal,yarotsky2019phase} or number of neurons \citep{shen2019deep,lu2020deep}.

Despite the vast amount of research on function approximation by neural networks, there are only a few papers studying the representational capacity of generative networks for approximating distributions. Let us compare our results with the most related works \citep{lee2017ability,bailey2018size,perekrestenko2020constructive,lu2020universal}. The paper \citep{lee2017ability} considered a special form of target distributions, which are push-forward measures of the source distributions via composition of Barron functions. These distributions, as they proved, can be approximated by deep generative networks. But it is not clear what probability distributions can be represented in the form they proposed. 
The works \citep{bailey2018size} and \citep{perekrestenko2020constructive} studied similar questions as ours. They showed that one can approximately transform low-dimensional distributions to high-dimensional distributions using neural networks in some cases. In \citep{bailey2018size}, the source and target distributions are restricted to uniform and Gaussian distributions. The paper \citep{perekrestenko2020constructive} proved the case that the source distribution is uniform and the target distribution has Lipschitz-continuous density function with bounded support. In this paper, we extend their results to a more general setting that the source distribution is absolutely continuous and the target only satisfies some moment conditions. Furthermore, we also show that the approximation orders in Wasserstein distances only depend on the intrinsic dimension of the target distribution, which is the first theoretical result of this kind.

In \citep{lu2020universal}, the authors showed that the gradients of neural networks, as transforms of distributions, are universal when the source and target distributions are of the same dimension. Their proof relies on the theory of optimal transport \citep{villani2008optimal}, which is only available between distributions of the same dimensions. Hence their approach cannot be simply extended to the case that the source and target distributions are of different dimensions. 
In contrast, we prove that neural networks are universal approximators for probability distributions even when the source distribution is one-dimensional. It means that neural networks can approximately transport low-dimensional distributions to high-dimensional distributions, which suggests some possible generalization of the optimal transport theory.

There is another line of works \citep{liang2018well,chen2020statistical} considering the estimation error of GANs. Similar to the error analysis of regression and classification, it was shown that the estimation error of GANs can be decomposed into three parts: statistical error, approximation errors of discriminator and generator. The statistical error can be bounded using statistical learning theory \citep{anthony2009neural,mohri2018foundations}. The discriminator approximation error can be dealt with using the function approximation theory of neural networks \citep{yarotsky2017error,shen2019deep}. Our results can be applied to estimate the generator approximation error.

\subsection{Notation and definition}

Let $\bN:=\{1,2,\dots\}$ be the set of natural numbers.
We denote the Euclidean distance of two points $x,y\in \bR^d$ by $|x-y|$. The ReLU function \citep{nair2010rectified} is denoted by $\sigma(x) := \max\{x,0\}$. For two probability measures $\mu$ and $\nu$, $\mu \perp \nu$ denotes that $\mu$ and $\nu$ are singular, $\mu \ll \nu$ denotes that $\mu$ is absolutely continuous with respect to $\nu$ and in this case the Radon–Nikodym derivative is denoted by $d\mu/d\nu$. We say $\mu$ is absolutely continuous if it is absolutely continuous with respect to the Lebesgue measure, which is equivalent to the statement that $\mu$ has probability density function.

\begin{definition}[Push-forward measure]
Let $\nu$ be a measure on $\bR^m$ and $\phi:\bR^m\rightarrow\bR^d$ a measurable mapping, where $m,d\in\bN$. The push-forward measure $\phi_\# \nu$ of a measurable set $A\subseteq \bR^d$ is defined as $\phi_\# \nu(A) := \nu(\phi^{-1}(A))$.
\end{definition}

\begin{definition}[Neural networks]
Let $L,N_0,\dots,N_{L+1} \in \bN$. A ReLU neural network with $L$ hidden layers is a collection of mapping $\phi:\bR^{N_0} \mapsto \bR^{N_{L+1}}$ of the form
$$
\phi(x) = A_L \circ \sigma \circ A_{L-1} \circ \cdots \circ \sigma \circ A_0 (x) \quad x\in \bR^{N_0},
$$
where $\sigma$ is applied element-wisely, $A_\ell(y) = M_\ell y+b_\ell$ is affine with $M_\ell \in \bR^{N_{\ell+1}\times N_\ell}$, $b_\ell\in \bR^{N_{\ell+1}}$, $\ell=0,\dots,L$. The quantities $W=\max_{\ell=1,\dots,L} N_\ell$ and $L$ are called the width and depth of the neural network, respectively. When the input and output dimensions are clear from the context, we denote $\cN\cN(W,L)$ as the set of functions that can be represented by ReLU neural networks with width at most $W$ and depth at most $L$.
\end{definition}

\begin{definition}[Covering number]
Given a set $S\subseteq \bR^d$, the $\epsilon$-covering number $\cN_\epsilon(S)$ is the smallest integer $m$ such that there exists $x_1,\dots,x_m\in \bR^d$ satisfying $S\subseteq \cup_{j=1}^m B(x_j,\epsilon)$, where $B(x,\epsilon)$ is the open ball of radius $\epsilon$ around $x$. For a probability measure $\mu$ on $\bR^d$, the $(\epsilon,\delta)$-covering number of $\mu$ is defined as
$$
\cN_\epsilon(\mu,\delta) := \inf \{\cN_\epsilon(S): \mu(S)\ge 1-\delta \}.
$$
\end{definition}

In this paper, we consider three classes of ``distances'' on probability distributions:
\begin{itemize}
	
	\item For $p\in [1,\infty)$, the $p$-th Wasserstein distance between two probability measures on $\bR^d$ is the optimal transportation cost defined as
	$$
	\cW_p(\mu,\nu) := \inf_{\gamma\in \prod(\mu,\nu)} (\bE_{\gamma(x,y)} |x-y|^p)^{1/p},
	$$
	where $\prod(\mu,\nu)$ denotes the set of all joint probability distributions $\gamma(x,y)$ whose marginals are respectively $\mu$ and $\nu$. A distribution $\gamma\in \prod(\mu,\nu)$ is called a coupling of $\mu$ and $\nu$. There always exists an optimal coupling that achieves the infimum \citep{villani2008optimal}. The Kantorovich-Rubinstein duality gives an alternative definition of $\cW_1$:
	$$
	\cW_1(\mu,\nu) = \sup_{\|f \|_{\Lip} \le 1} \bE_{\mu} [f] - \bE_{\nu} [f].
	$$
	This duality is used in Wasserstein GAN \citep{arjovsky2017wasserstein} to estimate the distance between the target and generated distributions. More generally, $\cW_p$ can be estimated by certain Besov norms of negative smoothness under some conditions \citep{weed2019estimation}.
	
	\item Let $\cH_K$ be a reproducing kernel Hilbert space (RKHS) with kernel $K:\bR^d \times \bR^d \to \bR$ \citep{aronszajn1950theory,berlinet2011reproducing}. The maximum mean discrepancy (MMD) between two probability distributions on $\bR^d$ is defined by \citep{gretton2012kernel}:
	$$
	\MMD(\mu,\nu) := \sup_{\|f\|_{\cH_K}\le 1} \bE_{\mu} [f] - \bE_{\nu} [f].
	$$
	Note that MMD and the Wasserstein distance $\cW_1$ are special cases of integral probability metrics \citep{muller1997integral}.
	
	\item The $f$-divergences, introduced by \citep{ali1966general} and \citep{csiszar1967information}, can be defined for all convex functions $f:(0,\infty) \to \bR$ with $f(1)=0$ as follows: Given two probability distributions $\mu,\nu$ that are absolutely continuous with respect to some probability measure $\tau$, let their Radon-Nikodym derivatives be $p=d\mu/d\tau$ and $q=d\nu/d\tau$. Then the $f$-divergence of $\mu$ from $\nu$ is defined as
	\begin{align*}
	D_f(\mu \| \nu) &:= \int_{\bR^d} f\left(\frac{p(x)}{q(x)} \right) q(x) d\tau(x) \\
	&= \int_{q>0} f\left(\frac{p(x)}{q(x)} \right) q(x) d\tau(x) +  f^*(0)\mu(q=0),
	\end{align*}
	where we denote $f(0):=\lim_{t\downarrow 0}f(t)$, $f^*(0):=\lim_{t\to \infty} f(t)/t$ and we adopt the convention that $f(\frac{p(x)}{q(x)}) q(x)=0$ if $p(x)=q(x)=0$, and $f(\frac{p(x)}{q(x)}) q(x)= f^*(0)p(x)$ if $q(x)=0$ and $p(x)\neq 0$. It can be shown that the definition is independent of the choice of $\tau$ and hence we can always choose $\tau = (\mu+\nu)/2$. 
	
\end{itemize}

\section{Approximation capacity of generative networks}

\subsection{Approximation in Wasserstein distances}

Given a source probability distribution $\nu$ on $\bR^m$, the objective of deep generative models is to find a deep neural network $\phi:\bR^m \to \bR^d$ such that the push-forward measure $\phi_\# \nu$ is close to the unknown target distribution $\mu$ on $\bR^d$ under certain metric, so that we can generate new samples using $\phi$. For instance, Wasserstein GAN \citep{arjovsky2017wasserstein} tries to compute a transform $\phi \in \cN\cN(W,L)$ that minimizes $\cW_1(\mu,\phi_\# \nu)$. 
In this work, we study approximation capacity of deep generative models, i.e., how well $\phi_\# \nu$ can approximate the target distribution $\mu$. Specifically, our goal is to estimate the quantity
$$
\cW_p(\mu,\cN\cN(W,L)_\# \nu):= \inf_{\phi \in \cN\cN(W,L)} \cW_p(\mu, \phi_\# \nu),
$$
where $\nu$ is an absolutely continuous probability distribution.

The basic idea is depicted as follows. To bound the approximation error $\cW_p(\mu,\cN\cN(W,L)_\#\nu)$, we first approximate the target distribution $\mu$ by a discrete probability measure $\gamma$, and then construct a neural network $\phi$ such that the push-forward measure $\phi_\# \nu$ is close to the discrete distribution $\gamma$. By the triangle inequality for Wasserstein distances, one has
\begin{align*}
\cW_p(\mu,\cN\cN(W,L)_\#\nu) 
\le& \cW_p(\mu,\gamma) + \cW_p(\gamma,\cN\cN(W,L)_\#\nu) \\
\le&  \cW_p(\mu,\gamma) + \sup_{\xi\in \cP(n)} \cW_p(\xi,\cN\cN(W,L)_\#\nu),
\end{align*}
where $\gamma \in \cP(n)$ and $\cP(n)$ is the set of all discrete probability measures supported on at most $n$ points, that is,
$$
\cP(n) := \left\{ \sum _{i=1}^n p_i\delta_{x_i}: \sum_{i=1}^n p_i=1, p_i\ge 0, x_i\in \bR^d \right\}.
$$
Taking the infimum over all $\gamma \in \cP(n)$, we get
\begin{equation}\label{upper bound}
\begin{aligned}
&\cW_p(\mu,\cN\cN(W,L)_\#\nu)
\le& \cW_p(\mu,\cP(n)) + \sup_{\xi\in \cP(n)} \cW_p(\xi,\cN\cN(W,L)_\#\nu),
\end{aligned}
\end{equation}
where $\cW_p(\mu,\cP(n)):=\inf_{\gamma\in \cP(n)}\cW_p(\mu,\gamma)$ measures the distance between $\mu$ and discrete distributions in $\cP(n)$. We will show that the second term vanishes as long as the width $W$ and depth $L$ of the neural network in use are sufficiently large (Lemma \ref{app discrete measure}). Consequently, our approximation problem is reduced to the estimation of the approximation error $\cW_p(\mu,\cP(n))$. We study the case that the target distribution $\mu$ has finite absolute $q$-moment
$$
M_q(\mu) := \left( \int_{\bR^d} |x|^q d\mu(x) \right)^{1/q} <\infty, \quad 1\le q<\infty.
$$ 
The main result can be summarized as follows.

\begin{theorem}\label{app finite moment}
Let $p\in [1,\infty)$ and $\nu$ be an absolutely continuous probability distribution on $\bR$. Assume that $\mu$ is a probability distribution on $\bR^d$ with finite absolute $q$-moment $M_q(\mu)<\infty$ for some $q>p$. Then, for any $W\ge 7d+1$ and $L\ge 2$,
$$
\cW_p(\mu,\cN\cN(W,L)_\# \nu) \le C (M_q^q(\mu)+1)^{1/p}
\begin{cases}
(W^2L)^{-1/d}, &q>p+p/d  \\
(W^2L)^{-1/d}(\log_2 W^2L)^{1/d}, &p<q\le p+p/d
\end{cases}
$$
where $C$ is a constant depending only on $p$, $q$ and $d$.
\end{theorem}

Note that the number of parameters of a neural network with width $W$ and depth $L$ is $n(W,L)=\Theta(W^2L)$, hence the theorem upper bounds the approximation error by the number of parameters. Although we restrict the source distribution $\nu$ to be one-dimensional, the result can be easily generalized to absolutely continuous distributions on $\bR^m$ such as multivariate Gaussian and uniform distributions. It can be done simply by projecting these distributions to one-dimensional distributions using linear mappings (the projection can be realized on the first layer of neural network). An interesting consequence of Theorem \ref{app finite moment} is that we can approximate high-dimensional distributions by low-dimensional distributions if we use neural networks as transport maps.

In generative adversarial network, the Wasserstein distance $\cW_1(\mu,\phi_\# \nu)$ is estimated by a discriminator parameterized by a neural network $\cF$:
$$
d_\cF(\mu,\phi_\# \nu) := \sup_{f\in \cF} \bE_{\mu} [f] - \bE_{\phi_\# \nu} [f].
$$
The discriminative network $\cF$ is often regularized (by weight clipping or other methods) so that the Lipschitz constant of any $f\in \cF$ is bounded by some constant $M>0$. For such network, we have
$$
d_\cF(\mu,\phi_\# \nu) \le M \cW_1(\mu,\phi_\# \nu).
$$
Hence, Theorem \ref{app finite moment} also gives upper bounds on the neural network distance $d_\cF$ used in Wasserstein GANs.

Notice that the bounds in Theorem \ref{app finite moment} suffer from the curse of dimensionality. In practice, the target distribution $\mu$ usually has certain low-dimensional structure, which can help us lessen the curse of dimensionality. To utilize this kind of structures, we introduce a notion of dimension of a measure, which is based on the concept of covering number.

\begin{definition}[Dimension]\label{dimension}
For a probability measure $\mu$ and $1\le p<\infty$, we define the upper and lower dimensions of $\mu$ as
\begin{align*}
s^*_p(\mu) &:= \limsup_{\epsilon \to 0} \frac{\log_2 \cN_\epsilon(\mu,\epsilon^p)}{-\log_2 \epsilon}, \\
s_*(\mu) &:= \lim_{\delta\to 0} \liminf_{\epsilon \to 0} \frac{\log_2 \cN_\epsilon(\mu,\delta)}{-\log_2 \epsilon},
\end{align*}
where $\cN_\epsilon(\mu,\delta)$ is the $(\epsilon,\delta)$-covering number of $\mu$.
\end{definition}

We make several remarks on the definition. Since $\cN_\epsilon(\mu,\delta)$ increases as $\delta$ decreases, the limit in the definition of lower dimension always exists. The monotonicity of $\cN_\epsilon(\mu,\delta)$ also implies that $s_*(\mu) \le s_p^*(\mu)\le s_q^*(\mu)$ for any $1\le p<q<\infty$. The lower dimension $s_*(\mu)$ is the same as the so-called lower Wasserstein dimension in \citep{weed2019sharp}, which was also introduced by \citep{young1982dimension} in dynamical systems. But our upper dimension $s^*_p(\mu)$ is different from the upper Wasserstein dimension in \citep{weed2019sharp}. More precisely, our upper dimension is slightly smaller than the upper Wasserstein dimension in some cases, hence leads to a better approximation order in Theorem \ref{app low dim by discrete}. 

To make it easier to interpret our results, we note that $s_*(\mu)$ and $s_p^*(\mu)$ can be bounded from below and above by the well known Hausdorff dimension and Minkowski dimension respectively (see \citep{falconer1997techniques,falconer2004fractal} for instance). \begin{proposition}\label{dimension relation} 
For any $1\le p<q<\infty$,
$$
\dim_H(\mu) \le s_*(\mu) \le s_p^*(\mu) \le s_q^*(\mu) \le \dim_M(\mu),
$$
where $\dim_H(\mu)$  and $\dim_M(\mu)$ are the Hausdorff and Minkowski dimensions, respectively.
\end{proposition}
This proposition indicates that our concepts of dimensions can capture geometric property of the distribution. The four dimensions above can all be regarded as intrinsic dimensions of distributions. For example, if $\mu$ is absolutely continuous with respect to the uniform distribution on a compact manifold of dimension $s$, then $\dim_H(\mu) = \dim_M(\mu)=s$, and hence we also have $s_*(\mu) = s_p^*(\mu) =s$ for all $p\in [1,\infty)$.

In the following theorem, we obtain an upper bound on the approximation error $\cW_p(\mu,\cN\cN(W,L)_\# \nu)$ in terms of the upper dimension of the target distribution $\mu$.

\begin{theorem}\label{app low dim}
Let $p\in [1,\infty)$ and $\nu$ be an absolutely continuous probability distribution on $\bR$. Suppose that $\mu$ is a probability measure on $\bR^d$ with finite absolute $q$-moment $M_q(\mu)<\infty$ for some $q>p$. If $s>s^*_{pq/(q-p)}(\mu)$, then for sufficiently large $W$ and $L$,
$$
\cW_p(\mu,\cN\cN(W,L)_\# \nu) \le Cd^{1/s}(M_q^p(\mu)+1)^{1/p} (W^2L)^{-1/s}
$$
where $C\le 384$ is an universal constant.
\end{theorem}

Notice that the approximation order only depends on the intrinsic dimension of the target distribution, and the bound grows only as $d^{1/s}$ for the ambient dimension $d$. It means that deep neural networks can overcome the course of dimensionality when approximating low-dimensional target distributions in high dimensional ambient spaces.

\subsection{Approximation in maximum mean discrepancy}

In this section, we apply our proof technique to the approximation in the maximum mean discrepancy. This distance was used as the loss function in GANs by \citep{dziugaite2015training,li2015generative}. Empirical evidences \citep{binkowski2018demystifying} show that MMD GANs require smaller discriminative networks than Wasserstein GANs. In the theoretical part, we will derive an approximation bound for the generative networks, where the decaying order is independent of the ambient dimension, in contrast with the approximation in Wasserstein distances.

Let $\cH_K$ be a RKHS with kernel $K:\bR^d \times \bR^d \to \bR$. For simplicity, we make two assumptions on the kernel: 

\begin{assumption}\label{kernel assumption 1}
The kernel $K$ is integrally strictly positive definite: for any finite non-zero signed Borel measure $\mu$ defined on $\bR^d$, we have
$$
\int_{\bR^d}\int_{\bR^d} K(x,y) d\mu(x) d\mu(y) > 0.
$$
\end{assumption}

\begin{assumption}\label{kernel assumption 2}
There exists a constant $\kappa>0$ such that 
$$
\sup_{x\in \bR^d} |K(x,x)| \le \kappa.
$$
\end{assumption}

These assumptions are satisfied by many commonly used kernels such as Gaussian kernel $K(x,y)= \exp(|x-y|^2/2\sigma^2)$, Laplacian kernel $K(x,y)=\exp(-\sigma|x-y|)$ and inverse multiquadric kernel $K(x,y) = (|x-y|+c)^{-1/2}$ with $c>0$. It was shown in \citep[Theorem 7]{sriperumbudur2010hilbert} that Assumption \ref{kernel assumption 1} is a sufficient condition for the kernel being characteristic: $\MMD(\mu,\gamma)=0$ if and only if $\mu=\gamma$, which implies that $\MMD$ is a metric on the set of all probability measures on $\bR^d$. We will use Assumption \ref{kernel assumption 2} to get approximation error bound for generative networks.

Let $\mu$ and $\nu$ be the target and source distributions respectively. As in the argument for Wasserstein distances, we have the following ``triangle inequalit'' for approximation error:
$$
\begin{aligned}
\MMD(\mu,\cN\cN(W,L)_\#\nu)
\le \MMD(\mu,\cP(n)) + \sup_{\xi\in \cP(n)} \MMD(\xi,\cN\cN(W,L)_\#\nu),
\end{aligned}
$$
where we denote $\MMD(\mu,\cN\cN(W,L)_\# \nu):= \inf_{\phi \in \cN\cN(W,L)} \MMD(\mu, \phi_\# \nu)$ and $\MMD(\mu,\cP(n)):=\inf_{\gamma\in \cP(n)}\MMD(\mu,\gamma)$. When the size of generative network is sufficiently large, for any given $\xi\in \cP(n)$, we can construct $g\in \cN\cN(W,L)$ such that $\MMD(\mu, g_\#\nu)$ is arbitrarily small. Hence, the second term vanishes. For the first term, we can approximate $\mu$ by its empirical distribution $\widehat{\mu}_n$. The following proposition, which is proved in \citep[Proposition 3.2]{lu2020universal}, gives a high-probability approximation bound for $\MMD(\mu,\widehat{\mu}_n)$.

\begin{proposition}\label{app by emp MMD}
Suppose the kernel satisfies Assumption \ref{kernel assumption 1} and \ref{kernel assumption 2}. Let $\widehat{\mu}_n = \frac{1}{n} \sum_{i=1}^n \delta_{X_i}$, where $X_i$ are i.i.d. samples from probability distribution $\mu$. Then, for any $\tau>0$, with probability at least $1-2e^{-\tau}$,
$$
\MMD(\mu, \widehat{\mu}_n) \le \frac{2\kappa^{1/4}}{\sqrt{n}} + \frac{3\sqrt{2\tau} \kappa^{1/4}}{\sqrt{n}}.
$$
\end{proposition}

By choosing $\tau$, we can upper bound $\MMD(\mu,\cP(n))$ and hence get an estimate on the approximation error of generative networks. The result is summarized in the next theorem. 

\begin{theorem}\label{app in MMD}
Suppose the kernel satisfies Assumption \ref{kernel assumption 1} and \ref{kernel assumption 2}. Let $\nu$ be an absolutely continuous probability distribution on $\bR$, then for any probability distribution $\mu$ on $\bR^d$, $W\ge 7d+1$ and $L\ge 2$,
$$
\MMD(\mu,\cN\cN(W,L)_\#\nu) \le 160 \sqrt{d} \kappa^{1/4} (W^2L)^{-1/2}. 
$$
\end{theorem}

\subsection{Approximation in $f$-divergences}

This section considers the approximation capacity of generative networks in $f$-divergences. These divergences are widely used in generative adversarial networks \citep{goodfellow2014generative,nowozin2016f}. For example, the vanilla GAN tries to minimize the Jensen–Shannon divergence of the generated distribution and the target distribution. However, it was shown by \citep{arjovsky2017towards} that the disjoint supports of these distributions cause instability and vanishing gradients in training the vanilla GAN. Nevertheless, for completeness, we discuss the approximation properties of generative networks in $f$-divergences and make a comparison with Wasserstein distances and MMD. Our discussions are based on the following proposition proved in the appendix.

\begin{proposition}\label{prop:f-div}
Assume that $f:(0,\infty)\to \bR$ is a strictly convex function with $f(1)=0$. If $\mu \perp \gamma$, then $D_f(\mu \| \gamma) = f(0)+f^*(0)>0$ is a constant.
\end{proposition}

Suppose that $f$ satisfies the assumption of Proposition \ref{prop:f-div}. Let $\mu$ and $\nu$ be target and source distributions on $\bR^d$ and $\bR^m$ respactively. Let $\phi:\bR^m \to \bR^d$ be a ReLU neural network.
We argue that to approximate $\mu$ by $\phi_\#\nu$ in $f$-divergences,
the dimension of $\nu$ should be no less than the intrinsic dimension of $\mu$. 

If $m<d$ and $\mu$ is absolutely continuous with respect to the Lebesgue measure, then $\mu\perp \phi_\#\nu$ and hence $D_f(\mu \| \phi_\#\nu)$ is a constant, which means we cannot approximate the target distribution $\mu$ by $\phi_\#\nu$.
More generally, we can consider the target distributions $\mu$ that are absolutely continuous with respect to the Riemannian measure \citep{pennec2006intrinsic} on some Riemannian manifold $\cM$ with dimension $s\le d$, which is a widely used assumption in applications. If $m<s$, then $\phi_\#\nu$ is supported on a manifold $\cN$ whose dimension is less than $s$ and the intersection $\cM\cap\cN$ has zero Riemannian measure on $\cM$. It implies that $\mu$ and $\phi_\# \nu$ are singular and hence $D_f(\mu \| \phi_\#\nu)$ is a positive constant. Therefore, in order to approximate $\mu$ in $f$-divergence, it is necessary that $m\ge s$.

Even when $m=s$, there still exists target distribution $\mu$ that cannot be approximated by ReLU neural networks. As an example, consider the case that $m=s=1$, $\nu$ is the uniform distribution on $[0,1]$ and $\mu$ is the uniform distribution on the unit circle $S^1\subseteq\bR^2$. Since the ReLU network $\phi:\bR \to \bR^2$ is a continuous piecewise linear function, $\phi([0,1])$ must be a union of line segments. Therefore, the intersection of $\phi([0,1])$ and the unit circle contains at most finite points, and thus its $\mu$-measure is zero. Hence, $\mu$ and $\phi_\#\nu$ are always singular and $D_f(\mu \| \phi_\#\nu)$ is a positive constant, no matter how large the network size is. In this example, it is not really possible to find any meaningful $\phi$ by minimizing $D_f(\mu \| \phi_\#\nu)$ using gradient decent methods, because the gradient always vanishes. A more detailed discussion of this phenomenon can be found in \citep{arjovsky2017towards}. 

On the other hand, a positive gap between two distributions in $f$-divergence does not necessarily mean that the distributions have gap in all aspects. In the above example of unit circle, we can actually choose a $\phi$ such that $\phi([0,1])$ is arbitrarily close to the unit circle in Euclidian distance, provided that the size of the network is sufficiently large. For such a $\phi$, the push-forward distribution $\phi_\#\nu$ and the target distribution $\mu$ generate similar samples, but their $f$-divergence is still $f(0)+f^*(0)$. This inconsistency shows that $f$-divergences are generally less adequate as metrics for the task of generating samples.

In summary, in order to approximate the target distribution in $f$-divergences, the dimension of the source distribution cannot be less than the intrinsic dimension of the target distribution. Even when the dimensions of the target distribution and the source distribution are the same, there exist some regular target distributions that cannot be approximated in $f$-divergences. In contrast, Theorem \ref{app finite moment} and \ref{app in MMD} show that we can use one-dimensional source distributions to approximate high-dimensional target distributions in Wasserstein distances and MMD, and the finite moment condition is already sufficient. It suggests that, from an approximation point of view, Wasserstein distances and MMD are more adequate as metrics of distributions for generative models.

\section{Proofs of main theorems}

\subsection{Approximation of discrete distributions}

In this section, we show how to use neural networks to approximate discrete distributions. For convenience, we denote $\cS^d(z_0,\dots,z_{N+1})$ as the set of all continuous piecewise linear functions $f:\bR \to \bR^d$ which have breakpoints only at $z_0<z_1<\dots<z_N<z_{N+1}$ and are constant on $(-\infty,z_0)$ and $(z_{N+1},\infty)$. The following lemma is an extension of the result in \citep{daubechies2021nonlinear}. The proof can be found in the appendix.

\begin{lemma}\label{CPwL}
Suppose that $W\ge 7d+1$, $L\ge 2$ and $N\le (W-d-1)\lfloor \frac{W-d-1}{6d} \rfloor \lfloor\frac{L}{2}\rfloor$. Then for any $z_0<z_1<\dots<z_N<z_{N+1}$, we have $\cS^d(z_0,\dots,z_{N+1}) \subseteq \cN\cN(W,L)$.
\end{lemma}

The lemma shows that if $N=\cO(W^2L/d)$, we have $\cS^d(z_0,\dots,z_{N+1}) \subseteq \cN\cN(W,L)$. We remark that the construction in this lemma is asymptotically optimal in the sense that if $\cS^d(z_0,\dots,z_{N+1}) \subseteq \cN\cN(W,L)$ for some $W,L\ge 2$, then the condition $N=\cO(W^2L/d)$ is necessary. To see this, we consider the function $F(\theta) := (f_\theta(z_0),\dots, f_\theta(z_{N+1}))$, where $f_\theta\in\cN\cN(W,L)$ is a ReLU neural network with parameters $\theta$. Let $n(W,L)$ be the number of parameters of the neural network $\cN\cN(W,L)$. By the assumption that $\cS^d(z_0,\dots,z_{N+1}) \subseteq \cN\cN(W,L)$, $F: \bR^{n(W,L)}\to \bR^{d(N+2)}$ is surjective and hence the Hausdorff dimension of $F(\bR^{n(W,L)})$ is $d(N+2)$. Since $F(\theta)$ is a piecewise multivariate polynomial of $\theta$, it is Lipschitz continuous on any bounded balls. It is well-known that Lipschitz maps do not increase Hausdorff dimension (see \citep[Theorem 2.8]{evans2015measure}). Since $F(\bR^{n(W,L)})$ is a countable union of images of bounded balls, its Hausdorff dimension is at most $n(W,L)$, which implies $d(N+2)\le n(W,L)$. Because of $n(W,L)=(L-1)W^2 +(L+d+1)W+d$, we have $N=\cO(W^2L/d)$.

The next two lemmas show that we can approximate $\mu\in \cP(n)$ arbitrarily well in Wasserstein distances and MMD using neural networks $\cN\cN(W,L)$ if $n=\cO(W^2L/d)$.

\begin{lemma}\label{app discrete measure}
Suppose that $W\ge 7d+1$, $L\ge 2$ and $p\in [1,\infty)$. Let $\nu$ be an absolutely continuous probability distribution on $\bR$. If $n\le \frac{W-d-1}{2} \lfloor \frac{W-d-1}{6d} \rfloor \lfloor \frac{L}{2} \rfloor +2$, then for any $\mu\in \cP(n)$,
$$
\cW_p(\mu,\cN\cN(W,L)_\#\nu) =0.
$$
\end{lemma}
\begin{proof}
Without loss of generality, we assume that $m:=n-1\ge 1$ and $\mu=\sum_{i=0}^{m} p_i \delta_{x_i}$ with $p_i>0$ for all $0\le i\le m$. For any $\epsilon$ that satisfies $0<\epsilon< (mp_i)^{1/p}|x_i-x_{i-1}|$ for all $i=1,\dots,m$, we are going to construct a neural network $\phi \in \cN\cN(W,L)$ such that $\cW_p(\mu,\phi_\# \nu) \le \epsilon$.

By the absolute continuity of $\nu$, we can choose $2m$ points
$$
z_{1/2} < z_1 < z_{3/2} < \dots < z_{m-1/2}< z_m
$$
such that
\begin{align*}
&\nu((-\infty,z_{1/2})) = p_0, \\
&\nu((z_{i-1/2},z_{i})) = \frac{\epsilon^p}{m|x_i-x_{i-1}|^p}, &&1\le i\le m,\\
&\nu((z_i,z_{i+1/2})) = p_i -\frac{\epsilon^p}{m|x_i-x_{i-1}|^p}, &&1\le i\le m-1,\\
&\nu((z_m,\infty)) = p_m -\frac{\epsilon^p}{m|x_m-x_{m-1}|^p}.
\end{align*}
We define the continuous piecewise linear function $\phi :\bR \to \bR^d$ by
$$
\phi(z) := 
\begin{cases}
x_0  &z \in (-\infty, z_{1/2}), \\
\frac{z_i-z}{z_i-z_{i-1/2}}x_{i-1} + \frac{z-z_{i-1/2}}{z_i-z_{i-1/2}}x_{i} &z\in [z_{i-1/2},z_i), \\
x_i  & z\in [z_{i},z_{i+1/2}), \\
x_m  &z\in [z_m,\infty).
\end{cases}
$$
Since $\phi \in \cS^d(z_{1/2},\dots, z_m)$ has $2m= 2n-2\le (W-d-1) \lfloor \frac{W-d-1}{6d} \rfloor \lfloor \frac{L}{2} \rfloor+2$ breakpoints, by Lemma \ref{CPwL}, $\phi \in \cN\cN(W,L)$.

In order to estimate $\cW_p(\mu,\phi_\# \nu)$, let us denote the line segment joining $x_{i-1}$ and $x_i$ by $\cL_i :=\{(1-t)x_{i-1} + tx_i\in \bR^d: 0<t\le 1 \}$. Then $\phi_\#\nu$ is supported on $\cup_{i=1}^m \cL_i \cup \{x_0\}$ and $\phi_\#\nu(\{x_0\}) = p_0$, $\phi_\# \nu(\{x_i\}) =p_i-\frac{\epsilon^p}{m|x_i-x_{i-1}|^p}$, $\phi_\# \nu(\cL_i) =p_i$ for $i=1,\dots,m$. By considering the sum of product measures
$$
\gamma = \delta_{x_0} \times \phi_\#\nu|_{\{x_0\}} + \sum_{i=1}^{m} \delta_{x_i} \times \phi_\#\nu|_{L_i},
$$
which is a coupling of $\mu$ and $\phi_\#\nu$, we have
\begin{align*}
\cW_p^p(\mu,\phi_\#\nu) &\le \int_{\bR^d \times \bR^d} |x-y|^p d\gamma(x,y) \\
&= \sum_{i=1}^m \int_{\cL_i \setminus \{x_i\}} |x_i -y|^p d\phi_\#\nu(y) \\
&\le \sum_{i=1}^m |x_i-x_{i-1}|^p \phi_\#\nu(\cL_i \setminus \{x_i\}) \\
&= \epsilon^p.
\end{align*}
Let $\epsilon\to0$, we have $\cW_p^p(\mu,\cN\cN(W,L)_\#\nu)=0$, which completes the proof.
\end{proof}

\begin{lemma}\label{app discrete measure MMD}
Suppose $W\ge 7d+1$, $L\ge 2$ and the kernel satisfies Assumption \ref{kernel assumption 1} and \ref{kernel assumption 2}. Let $\nu$ be an absolutely continuous probability distribution on $\bR$. If $n\le \frac{W-d-1}{2} \lfloor \frac{W-d-1}{6d} \rfloor \lfloor \frac{L}{2} \rfloor +2$, then for any $\mu\in \cP(n)$,
$$
\MMD(\mu,\cN\cN(W,L)_\#\nu) =0.
$$
\end{lemma}
\begin{proof}
Similar to the proof of Lemma \ref{app discrete measure}, we assume that $m:=n-1\ge 1$ and $\mu=\sum_{i=0}^{m} p_i \delta_{x_i}$ with $p_i>0$ for all $0\le i\le m$. For any $\epsilon<m \min_{1\le i\le m} p_i$, we choose $2m$ points $z_{1/2} < z_1 < z_{3/2} < \dots < z_{m-1/2}< z_m$ such that 
\begin{align*}
&\nu((-\infty,z_{1/2})) = p_0, \\
&\nu((z_{i-1/2},z_{i})) = \frac{\epsilon}{m}, &&1\le i\le m,\\
&\nu((z_i,z_{i+1/2})) = p_i -\frac{\epsilon}{m}, &&1\le i\le m-1,\\
&\nu((z_m,\infty)) = p_m -\frac{\epsilon}{m}.
\end{align*}
Define $\phi :\bR \to \bR^d$ as in the proof of Lemma \ref{app discrete measure}, then $\phi \in \cN\cN(W,L)$.

It remains to estimate $\MMD(\mu, \phi_\#\nu)$. Recall that we denote the line segment $\cL_i :=\{(1-t)x_{i-1} + tx_i\in \bR^d: 0<t\le 1 \}$, then $\phi_\#\nu(\{x_0\}) = p_0$, $\phi_\# \nu(\{x_i\}) =p_i-\epsilon/m$ and $\phi_\# \nu(\cL_i) =p_i$ for $i=1,\dots,m$. Thanks to \citep[Theorem 1]{sriperumbudur2010hilbert}, one has
$$
\MMD(\mu, \phi_\#\nu) = \left\| \int_{\bR^d} K(\cdot,x)d\mu(x) - \int_{\bR^d} K(\cdot,x)d\phi_\# \nu(x) \right\|_{\cH_K},
$$
where the integrals are defined in Bochner sense. Hence,
\begin{align*}
\MMD(\mu, \phi_\#\nu) &= \left\| \frac{\epsilon}{m}\sum_{i=1}^m K(\cdot,x_i) - \sum_{i=1}^m \int_{\cL_i \setminus \{x_i\}} K(\cdot,x)d\phi_\# \nu(x) \right\|_{\cH_K} \\
&\le \sum_{i=1}^m \frac{\epsilon}{m} \|K(\cdot,x_i)\|_{\cH_K} + \int_{\cL_i \setminus \{x_i\}} \| K(\cdot,x) \|_{\cH_K} d\phi_\# \nu(x)  \\
&\le 2\sqrt{\kappa} \epsilon,
\end{align*}
where we use $\| K(\cdot,x) \|_{\cH_K} = \sqrt{k(x,x)}\le \sqrt{\kappa} $ by Assumption \ref{kernel assumption 2} in the last inequality. Letting $\epsilon\to 0$ finishes the proof.
\end{proof}

\subsection{Proof of Theorem \ref{app finite moment}}

We will use the following lemma to estimate the Wasserstein distances of two distributions.

\begin{lemma}\label{Wp estimate method}
If two probability measures $\mu$ and $\gamma$ on $X\subseteq \bR^d$ can be decomposed into non-negative measures as $\mu = \sum_{j\ge 1} \mu_j$ and $\gamma = \sum_{j\ge 1} \gamma_j$ such that $\mu_j(X) = \gamma_j(X)$ for all $j\ge 1$, then
$$
\cW_p^p(\mu,\gamma) \le \sum_{j\ge 1} \mu_j(X) \cW_p^p \left(\frac{\mu_j}{\mu_j(X)}, \frac{\gamma_j}{\gamma_j(X)} \right).
$$
In particular, if the support of $\mu$ can be covered by $n$ balls $\{B(x_j,r)$: $j=1,\dots,n\}$, then there exists $c_j\ge 0$ such that $\sum_{j=1}^n c_j =1$ and 
$$
\cW_p \left(\mu, \sum_{j=1}^{n} c_j \delta_{x_j} \right) \le r.
$$
\end{lemma}
\begin{proof}
Let $\xi_j$ be the optimal coupling of $\frac{\mu_j}{\mu_j(X)}$ and $\frac{\gamma_j}{\gamma_j(X)}$, then it is easy to check that
$$
\xi = \sum_{j\ge 1} \mu_j(X) \xi_j
$$
is a coupling of $\mu$ and $\gamma$. Hence,
\begin{align*}
\cW_p^p(\mu,\gamma) &\le \int_X\int_X |x-y|^p d\xi(x,y) \\
&= \sum_{j\ge 1} \mu_j(X) \int_X\int_X |x-y|^p d\xi_j(x,y) \\
&= \sum_{j\ge 1} \mu_j(X) \cW_p^p \left(\frac{\mu_j}{\mu_j(X)}, \frac{\gamma_j}{\gamma_j(X)} \right).
\end{align*}

For the second part of the lemma, Let $A$ be the support of $\mu$, denote $A_1 := A \cap B(x_1,r)$ and $A_{j+1} := A\cap B(x_{j+1},r) \setminus (\cup_{i=1}^{j} A_i)$, then $\{A_j:j=1,\dots,n \}$ is a partition of $A$. This partition induces a decomposition of $\mu = \sum_{j=1}^n \mu|_{A_j}$.

Let $c_j = \mu(A_j)$, then $\sum_{j=1}^n c_j =1$ and if $c_j\neq 0$,
\begin{align*}
\cW_p^p (c_j^{-1} \mu|_{A_j}, \delta_{x_j}) \le c_j^{-1} \int\int |x-y|^p d\delta_{x_j}(x) d\mu|_{A_j}(y) 
\le r^p.
\end{align*}
By the first part of the lemma, we have
\begin{align*}
\cW_p^p \left(\mu, \sum_{j=1}^{n} c_j \delta_{x_j} \right) \le \sum_{j=1}^n c_j \cW_p^p (c_j^{-1} \mu|_{A_j}, \delta_{x_j}) 
\le \sum_{j=1}^n c_j r^p = r^p,
\end{align*}
which completes the proof.
\end{proof}

Using Lemma \ref{Wp estimate method}, we can give upper bounds of the approximation error $\cW_p(\mu,\cP(n))$ for distribution $\mu$ with finite moment.

\begin{theorem}\label{app finite moment by discrete}
Let $\mu$ be a probability distribution on $\bR^d$ with finite absolute $q$-moment $M_q(\mu)<\infty$ for some $q>p\ge 1$. Then for any $n\in \bN$, 
$$
\cW_p(\mu,\cP(n)) \le C_{p,q,d}(M_q^q(\mu)+1)^{1/p}
\begin{cases}
n^{-1/d}, &q>p+p/d  \\
(n/\log_2 n)^{-1/d}, &p<q\le p+p/d
\end{cases}
$$
where $C_{p,q,d}$ is a constant depending only on $p$, $q$ and $d$.
\end{theorem}

\begin{proof}
Let $B_0 = \{x\in \bR^d: |x|\le 1 \}$ and $B_j = (2^jB_0)\setminus (2^{j-1}B_{0})$ for $j\in \bN$, then $\{B_j\}_{j\ge 0}$ is a partition of $\bR^d$. For any $k\ge 0$, we denote $E_k:=\bR^d \setminus \cup_{j=0}^k B_j$. Let $\mu_j := \mu(B_j)^{-1} \mu|_{B_j}$ and $\tilde{\mu}_k:= \mu(E_k)^{-1} \mu|_{E_k}$, then for each $k\ge 0$, we can decompose $\mu$ as 
$$
\mu = \sum_{j=0}^k \mu(B_j) \mu_j + \mu(E_k) \tilde{\mu}_k.
$$
By Markov's inequality, we have
$$
\mu(B_j) \le \mu(|x|> 2^{j-1}) \le M_q^q(\mu)2^{-q(j-1)}, \quad j\ge 1.
$$
Furthermore, if $\mu(E_k)\neq 0$, 
\begin{align*}
\cW_p^p(\tilde{\mu}_k, \delta_0) &\le \mu(E_k)^{-1} \int_{|x|>2^k} |x|^p d\mu(x) \\
&\le \mu(E_k)^{-1} \int_{|x|>2^k} |x|^p \frac{|x|^{q-p}}{2^{k(q-p)}} d\mu(x)\\
&\le \mu(E_k)^{-1} M_q^q(\mu) 2^{-k(q-p)}.
\end{align*}

Thanks to \citep[Theorem 3.1]{verger2005covering}, the ball $\{x:|x|\le 2^j\}$ can be covered by at most $C_d r^{-d}$ balls with radius $2^jr$, where $C_d \le Cd^2$ for some constant $C>0$. Let $r_0,\dots,r_k$ be $k+1$ positive numbers, then each $B_j \subseteq \{x:|x|\le 2^j\}$ can be covered by at most $n_j := \lceil C_d r_j^{-d} \rceil$ balls with radius $2^jr_j$. We denote the collection of the centers of these balls by $\Lambda_j$. By Lemma \ref{Wp estimate method}, there exists a probability measure $\gamma_j$ of the form
$$
\gamma_j = \sum_{x\in \Lambda_j} c_j(x) \delta_x
$$
such that $\cW_p(\mu_j,\gamma_j) \le 2^j r_j$.

We consider the probability distribution
$$
\gamma = \sum_{j=0}^k \mu(B_j) \gamma_j + \mu(E_k) \delta_0 \in \cP\left(1+\sum_{j=0}^k n_j \right).
$$
By Lemma \ref{Wp estimate method}, we have
\begin{align*}
\cW_p^p(\mu,\gamma) \le&\sum_{j=0}^k \mu(B_j) \cW_p^p(\mu_j,\gamma_j) + \mu(E_k) \cW_p^p(\tilde{\mu}_k, \delta_0) \\
\le& r_0^p + M_q^q(\mu) \sum_{j=1}^k 2^{q-j(q-p)} r_j^p + M_q^q(\mu)2^{-k(q-p)}.
\end{align*}

Finally, if $q> p+p/d$, we choose $k = \lfloor \log_2 n \rfloor -1$ and $r_j = C_d^{1/d} 2^{(j-k)/d}$ for $0\le j\le k$. Then, $1+ \sum_{j=0}^k n_j = 1+ \sum_{j=0}^{k} 2^{k-j} =2^{k+1} \le n$, which implies $\gamma \in \cP(n)$, and
\begin{align*}
\cW_p^p(\mu,\gamma) &\le M_q^q(\mu) 2^{-k(q-p)} + C_d^{p/d} 2^{-pk/d} + 2^q C_d^{p/d} M_q^q(\mu) 2^{-pk/d} \sum_{j=1}^k 2^{-j(q-p-p/d)} \\
&\le \left(M_q^q(\mu) + C_d^{p/d} + C_d^{p/d} M_q^q(\mu) \frac{2^q}{2^{q-p-p/d}-1} \right) 2^{-pk/d} \\
&\le C_{p,q,d}^p (M_q^q(\mu)+1) n^{-p/d}.
\end{align*}
If $p<q\le p+p/d$, we choose $k=\lceil \frac{p}{d(q-p)} \log_2 n \rceil$ and $r_j = C_d^{1/d} m^{-1/d}$ for $0\le j\le k$, where $m:= \lfloor \frac{n-1}{k+1} \rfloor$ . Then we have $1+ \sum_{j=0}^k n_j = 1+ (k+1)\lfloor \frac{n-1}{k+1} \rfloor \le n$, which implies $\gamma \in \cP(n)$, and
\begin{align*}
\cW_p^p(\mu,\gamma) &\le M_q^q(\mu) 2^{-k(q-p)} + C_d^{p/d} m^{-p/d} + 2^q C_d^{p/d} M_q^q(\mu) m^{-p/d} \sum_{j=1}^k 2^{-j(q-p)} \\
&\le M_q^q(\mu) n^{-p/d} + C_d^{p/d}\left(1+ M_q^q(\mu) \frac{2^q}{2^{q-p}-1} \right) m^{-p/d} \\
&\le C_{p,q,d}^p (M_q^q(\mu)+1) (n/\log_2 n)^{-p/d}. \qedhere
\end{align*}
\end{proof}

The expected Wasserstein distance $\bE W_p(\mu,\hat{\mu}_n)$ between a probability distribution $\mu$ and its empirical distribution $\hat{\mu}_n$ has been studied extensively in statistics literature \citep{fournier2015rate,bobkov2019one,weed2019sharp,lei2020convergence}. It was shown in \citep{lei2020convergence} that, if $M_q(\mu)<\infty$, the convergence rate of $\bE W_p(\mu,\hat{\mu}_n)$ is $n^{-s(p,q,d)}$ with $s(p,q,d) = \min\{ \frac{1}{d}, \frac{1}{2p}, \frac{1}{p}-\frac{1}{q}\}$, ignoring the logarithm factors. Since $\hat{\mu}_n\in \cP(n)$, it is easy to see that $\cW_p(\mu,\cP(n))\le \bE \cW_p(\mu, \hat{\mu}_n)$. In Theorem \ref{app finite moment by discrete}, we construct a discrete measure that achieves the order $1/d\ge s(p,q,d)$, which is slightly better than the empirical measure in some situations.

By the triangle inequality (\ref{upper bound}), we can use Theorem \ref{app finite moment by discrete} and Lemma \ref{app discrete measure} to prove our main Theorem \ref{app finite moment}.

\begin{proof}[Proof of Theorem \ref{app finite moment}]
Inequality (\ref{upper bound}) says that, for any $n$,
\begin{align*}
\cW_p(\mu,\cN\cN(W,L)_\#\nu)
\le \cW_p(\mu,\cP(n)) + \sup_{\xi\in \cP(n)} \cW_p(\xi,\cN\cN(W,L)_\#\nu).
\end{align*}
If we choose $n= \frac{W-d-1}{2} \lfloor \frac{W-d-1}{6d} \rfloor \lfloor \frac{L}{2} \rfloor +2$, Lemma \ref{app discrete measure} implies that $\cW_p(\mu,\cN\cN(W,L)_\#\nu) \le \cW_p(\mu,\cP(n))$. By Theorem \ref{app finite moment by discrete}, it can be bounded by
$$
C (M_q^q(\mu)+1)^{1/p}
\begin{cases}
n^{-1/d}, &q>p+p/d  \\
(n/\log_2 n)^{-1/d}, &p<q\le p+p/d
\end{cases}
$$
for some constant $C>0$ depending only on $p$, $q$ and $d$.

Since $W\ge 7d+1$ and $L\ge 2$, a simple calculation shows $cW^2L/d \le n \le CW^2L/d$ with $c=1/384$ and $C=1/12$. Hence, $n^{-1/d}\le c^{-1}d^{1/d}(W^2L)^{-1/d} \le 2c^{-1}(W^2L)^{-1/d}$ and $(\log_2 n)^{1/d} \le (\log_2 W^2L)^{1/d}$, which give us the desired bounds.
\end{proof}

\subsection{Proof of Theorem \ref{app low dim}}

By the triangle inequality (\ref{upper bound}), Theorem \ref{app low dim} is a direct consequence of Lemma \ref{app discrete measure} and the upper bound of $\cW_p(\mu,\cP(n))$ in the following Theorem \ref{app low dim by discrete}. The proof follows the same argument as the proof of Theorem \ref{app finite moment}. We also give a lower bound in the following theorem, which indicates the tightness of the upper bound.

\begin{theorem}\label{app low dim by discrete}
Suppose that $1\le p<q< \infty$. Let $\mu$ be a probability measure on $\bR^d$ with finite absolute $q$-moment $M_q(\mu)<\infty$. If $s>s^*_{pq/(q-p)}(\mu)$, then for sufficiently large $n$,
$$
\cW_p(\mu,\cP(n)) \le (M_q^p(\mu)+1)^{1/p} n^{-1/s}.
$$
If $t<s_*(\mu)$, then there exists a constant $C_\mu$ depending on $\mu$ such that
$$
\cW_p(\mu,\cP(n)) \ge C_\mu n^{-1/t}.
$$
\end{theorem}

\begin{proof}
If $s>s^*_{pq/(q-p)}(\mu)$, then there exists $\epsilon_0>0$ such that, $\log_2 \cN_\epsilon(\mu,\epsilon^{pq/(q-p)}) < -s \log_2 \epsilon$ for all $\epsilon\in(0,\epsilon_0)$. For any $n>\epsilon_0^{-s}$, we set $\epsilon= n^{-1/s}<\epsilon_0$, then $\cN_\epsilon(\mu,\epsilon^{pq/(q-p)})<n$. 

By the definition of $(\epsilon,\delta)$-covering number, there exists $S$ with $\mu(S)\ge 1-\epsilon^{pq/(q-p)}$ such that $S$ is covered by at most $n-1\ge \cN_\epsilon(S)$ balls $B(x_j,\epsilon)$, $j=1,\dots,n-1$. Let $F_1= S\cap B(x_1,\epsilon)$ and $F_j = (S\cap B(x_j,\epsilon)) \setminus (\cup_{1\le i<j} F_i) $ for $2\le j\le n-1$, then $F_j\subseteq B(x_j,\epsilon)$ and $\{ F_j:1\le j\le n-1 \}$ is a partition of $S$.
	
We consider the probability distribution $\gamma = \mu(\bR^d \setminus S) \delta_0 + \sum_{j=1}^{n-1} \mu(F_j) \delta_{x_j} \in \cP(n)$. Let 
$$
\xi = \delta_0 \times \mu|_{\bR^d \setminus S} + \sum_{j=1}^{n-1} \delta_{x_j} \times \mu|_{F_j},
$$
then $\xi$ is a coupling of $\gamma$ and $\mu$, and we have
\begin{align*}
\cW_p^p(\mu,\gamma) &\le \int_{\bR^d \times \bR^d} |x-y|^p d\xi(x,y) \\
&= \int_{\bR^d\setminus S} |y|^pd\mu(y) + \sum_{j=1}^{n-1} \int_{F_j} |x_j -y|^p d\mu(y) \\
&\le \mu(\bR^d\setminus S)^{1-p/q} M_q^p(\mu) + \mu(S)\epsilon^p,
\end{align*}
where we use H\"older's inequality in the last step. Since $\mu(\bR^d\setminus S)\le \epsilon^{pq/(q-p)}$, we have
$$
\cW_p^p(\mu,\gamma) \le (M_q^p(\mu)+1) \epsilon^p = (M_q^p(\mu)+1) n^{-p/s}.
$$

The second part of the theorem was also proved in \citep{weed2019sharp}. If $t<s_*(\mu)$, there exists $\delta>0$ and $\epsilon_0>0$ such that $\cN_{\epsilon}(\mu,\delta) >\epsilon^{-t}$ for all $\epsilon\in(0,\epsilon_0)$.  For any $n>\epsilon_0^{-t}$, we set $\epsilon= n^{-1/t}<\epsilon_0$, then $\cN_\epsilon(\mu,\delta)>n$.  For any $\gamma = \sum_{i=1}^{n}p_i\delta_{x_i} \in \cP(n)$, let $S=\cup_{i=1}^n B(x_i,\epsilon)$, then $\mu(S)<1-\delta$ due to $\cN_\epsilon(\mu,\delta)>n$. This implies
$$
\mu\left( \left\{ y: \min_{1\le i\le n}|x_i-y|\ge \epsilon \right\} \right) \ge \delta.
$$
Hence, for any coupling $\xi$ of $\gamma$ and $\mu$,
\begin{align*}
\int_{\bR^d \times \bR^d} |x-y|^p d\xi(x,y) 
=& \int_{ \cup_{j=1}^n \{x_j\} \times\bR^d} |x-y|^p d\xi(x,y) \\
\ge& \int_{\bR^d} \min_{1\le i\le n}|x_i-y|^p d\mu(y) \\
\ge& \delta \epsilon^p = \delta n^{-p/t}.
\end{align*}
Take infimum in $\xi$ over all the couplings of $\mu$ and $\gamma$, we have $\cW_p(\mu,\gamma)\ge \delta^{1/p}n^{-1/t}$.
\end{proof}

We remark that \citep{weed2019sharp} gave similar upper bound on the expected error $\bE \cW_p(\mu,\hat{\mu}_n)$ of the empirical distribution $\hat{\mu}_n$. But the dimension they introduced is slightly large then $s^*_{pq/(q-p)}(\mu)$, hence our approximation order is better in some cases.

\subsection{Proof of Theorem \ref{app in MMD}}

The proof is similar to the proof of Theorem \ref{app finite moment}. By the ``triangle inequalit'', for any $n\in \bN$, 
$$
\begin{aligned}
\MMD(\mu,\cN\cN(W,L)_\#\nu)
\le \MMD(\mu,\cP(n)) + \sup_{\xi\in \cP(n)} \MMD(\xi,\cN\cN(W,L)_\#\nu).
\end{aligned}
$$
We choose $n= \frac{W-d-1}{2} \lfloor \frac{W-d-1}{6d} \rfloor \lfloor \frac{L}{2} \rfloor +2$, then Lemma \ref{app discrete measure MMD} and Proposition \ref{app by emp MMD} imply
$$
\MMD(\mu,\cN\cN(W,L)_\#\nu) \le \MMD(\mu,\cP(n)) \le \frac{8\kappa^{1/4}}{\sqrt{n}},
$$
where we set $\tau=2$ in Proposition \ref{app by emp MMD} to guarantee the existence of $\widehat{\mu}_n \in \cP(n)$ that satisfies the upper bound. Since $W\ge 7d+1$ and $L\ge 2$, it is easy to check that $n\ge cW^2L/d$ with $c=1/384$. Hence,
$$
\MMD(\mu,\cN\cN(W,L)_\#\nu) \le 160 \sqrt{d} \kappa^{1/4}(W^2L)^{-1/2}.
$$

\section{Conclusion}

In this paper, we study the approximation capacity of generative networks in three metrics: Wasserstein distances, MMD and $f$-divergences. For Wasserstein distances and MMD, we show that generative networks are universal approximators in the sense that, under mild conditions, they can approximately transform low-dimensional distributions to any high-dimensional distributions. The approximation bounds are obtained in terms of the width and depth of neural networks. We also show that the approximation orders in Wasserstein distances only depend on the intrinsic dimension of the target distribution. On the contrary, for $f$-divergences, it is impossibles to approximate the target distribution using neural networks, if the dimension of the source distribution is smaller than the intrinsic dimension of the target. 

One shortcoming of our analysis is that the weights in our network construction is unbounded. For example, in Lemma \ref{app discrete measure}, the weights of neural network diverge to infinity when the approximation error approaches zero. Using the space-filling
approach discovered in \citep{bailey2018size}, a recent paper \citep{perekrestenko2021high} estimated approximation bounds for neural networks with bounded weights, under the assumption that the source distribution is uniform and the target distribution has bounded support. It will be interesting to see whether their proof techniques can be applied to the general setting of this paper and combined with our analysis.

\section*{Acknowledgment}

The research of Y. Wang is supported by the HK RGC grant 16308518, the HK Innovation Technology Fund Grant  ITS/044/18FX and the Guangdong-Hong Kong-Macao Joint Laboratory for Data Driven Fluid Dynamics and Engineering Applications (Project 2020B1212030001).

\appendix

\renewcommand{\thesection}{\Alph{section}}

\section{Remaining proofs}
\subsection{Proof of Proposition \ref{dimension relation}}

We first recall the definition of Hausdorff and Minkowski dimensions.

\begin{definition}[Hausdorff and Minkowski dimensions]\label{Hausdorff Minkowski dimensions}
The $\alpha$-Hausdorff measure of a set $S$ is defined as
$$
\cH^\alpha(S):= \lim_{\epsilon\to 0} \inf \left\{ \sum_{j=1}^{\infty} r_j^\alpha: S\subseteq \cup_{j=1}^\infty B(x_j,r_j), r_j\le \epsilon \right\},
$$
where $B(x,r)$ is the open ball with center $x$ and radius $r$, and the Hausdorff dimension of $S$ is
$$
\dim_H(S) := \inf \{ \alpha: \cH^\alpha(S) =0 \}.
$$
The Minkowski dimension of $S$ is
$$
\dim_M(S) := \limsup_{\epsilon\to 0} \frac{\log_2 \cN_\epsilon(S)}{-\log_2 \epsilon},
$$
where $\cN_\epsilon(S)$ is the $\epsilon$-covering number of $S$.
The Hausdorff and Minkowski dimensions of a measure $\mu$ are defined respectively by
\begin{align*}
\dim_H(\mu) &:= \inf \{ \dim_H(S): \mu(S)=1 \}, \\
\dim_M(\mu) &:= \inf \{ \dim_M(S): \mu(S)=1 \}.
\end{align*}
\end{definition}

We first prove that $s_p^*(\mu) \le \dim_M(\mu)$. Observe that for any $p$ and any $S$ with $\mu(S)=1$, 
$$
\cN_\epsilon(\mu,\epsilon^p) \le \cN_\epsilon(\mu,0) \le \cN_\epsilon(S).
$$
A straightforward application of the definitions implies that $s_p^*(\mu) \le \dim_M(\mu)$.

For the inequality $\dim_H(\mu) \le s_*(\mu)$, we follow the idea in \citep{weed2019sharp}. By \citep[Proposition 10.3]{falconer1997techniques}, the Hausdorff dimension of $\mu$ can be expressed as
$$
\dim_H(\mu) = \inf \left\{ s\in \bR : \liminf_{r\to 0} \frac{\log_2 \mu(B(x,r))}{\log_2 r} \le s \mbox{ for } \mu \mbox{-a.e. } x \right\}.
$$
This implies for any $s< \dim_H(\mu)$ that
$$
\mu\left( \left\{ x: \exists r_x>0, \forall r\le r_x, \mu(B(x,r))\le r^s \right\} \right) \ge \mu\left( \left\{ x: \liminf_{r\to 0} \frac{\log_2 \mu(B(x,r))}{\log_2 r}> s \right\} \right) >0.
$$
Consequently (see the proof of \citep[Corollary 12.16]{graf2007foundations}), there exists $r_0>0$ and a compact set $K\subseteq \bR^d$ with $\mu(K)>0$ such that $\mu(B(x,r)) \le r^s$ for all $x\in K$ and all $r\le r_0$. 

For any $\delta<\mu(K)/2$ and any $S$ with $\mu(S) \ge 1-\delta$, we have $\mu(S\cap K) \ge \mu(K) - \mu(\bR^d \setminus S) \ge \mu(K)/2$. Observe that any ball with radius $\epsilon$ that intersects $S\cap K$ is contained in a ball $B(x,2\epsilon)$ with $x\in K$. Thus, $S\cap K$ can be covered by $\cN_{\epsilon}(S)$ balls with radius $2\epsilon$ and centers in $K$. If $2\epsilon\le r_0$, then each ball satisfies $\mu(B(x,2\epsilon)) \le (2\epsilon)^s$ and hence
$$
\cN_{\epsilon}(S) \ge (2\epsilon)^{-s} \mu(K)/2.
$$
Therefore, for all $\delta<\mu(K)/2$,
$$
\liminf_{\epsilon \to 0} \frac{\log_2 \cN_\epsilon(\mu,\delta)}{-\log_2 \epsilon} \ge \liminf_{\epsilon \to 0} \frac{-s(\log_2\epsilon+1) + \log_2 \mu(K) -1}{-\log_2 \epsilon} = s.
$$
Consequently, $s_*(\mu)\ge s$. Since $s< \dim_H(\mu)$ is arbitrary, we have $\dim_H(\mu) \le s_*(\mu)$.

\subsection{Proof of Proposition \ref{prop:f-div}}

By the convexity of $f$, the right derivative
$$
f'_+(t) := \lim_{\epsilon \downarrow 0} \frac{f(t+\epsilon)-f(t)}{\epsilon}
$$
always exists and is finite on $(0,\infty)$. Let
$$
\tilde{f}(t) := f(t) - f'_+(1)(t-1) \ge 0,
$$
then $\tilde{f}$ is strictly convex, decreasing on $(0,1)$ and increasing on $(1,\infty)$ with $\tilde{f}(1)=0$. It is easy to check that $f$ and $\tilde{f}$ induce the same divergence $D_f=D_{\tilde{f}}$. Hence, substituting by $\tilde{f}$ if necessary, we can always assume that $f$ is strictly convex with the global minimum $f(1)=0$.

By Lebesgue's decomposition theorem, we have
$$
\mu = \mu_a + \mu_s,
$$
where $\mu_a \ll \gamma$ and $\mu_s \perp \gamma$. A simple calculation shows that
$$
D_f(\mu \| \gamma) = \int_{\bR^d} f\left(\frac{ d\mu_a}{ d\gamma} \right) d \gamma +  f^*(0)\mu_s(\bR^d).
$$
Since $\mu$ and $\gamma$ are singular, we have $\mu_a=0$ and $\mu_s(\bR^d)=1$. Therefore, $D_f(\mu \| \gamma) = f(0)+f^*(0)>0$.

\subsection{Proof of Lemma \ref{CPwL}}

We follow the construction in \citep[Lemma 3.3 and 3.4]{daubechies2021nonlinear}. Recall that $\cS^d(z_0,\dots,z_{N+1})$ is the set of all continuous piecewise linear (CPwL) functions which have breakpoints only at $z_0<z_1<\dots<z_N<z_{N+1}$ and are constant on $(-\infty,z_0)$ and $(z_{N+1},\infty)$. It is easy to check that $\cS^d(z_0,\dots,z_{N+1})$ is a linear space. We denote the $dN$-dimensional linear subspace of $\cS^d(z_0,\dots,z_{N+1})$ that contains all functions which vanish outside $[z_0,z_{N+1}]$ by $\cS^d_0(z_0,\dots,z_{N+1})$. 

When $N=qW$ for some integers $q$ and $W$, we can construct a basis of $S^1_0(z_0,\dots,z_{N+1})$ as follows: for $1\le m\le q$ and $1\le j\le W$, let $h_{m,j}$ be the hat function which vanishes outside $[z_{jq-m},z_{jq+1}]$, takes the value one at $z_{jq}$ and is linear on each of the intervals $[z_{jq-m},z_{jq}]$ and $[z_{jq},z_{jq+1}]$. The breakpoint $z_{jq}$ is called the principal breakpoint of $h_{m,j}$. We order these hat functions by their leftmost breakpoints $z_{n-1}$ and rename them as $h_n$, $n=1,\dots N$, that is $h_n=h_{m,j}$ where $n-1=jq-m$. It is easy to check that $h_n$'s are a basis for $S^1_0(z_0,\dots,z_{N+1})$. The following lemma is a modification of \citep[Lemma 3.3]{daubechies2021nonlinear}. Recall that $\sigma(x) := \max\{x,0\}$ is the ReLU function.

\begin{lemma}\label{CPwL two layer}
For any breakpoints $z_0<\dots<z_{N+1}$ with $N=qW$, $q=\lfloor \frac{W}{6d} \rfloor$, $W\ge 6d$, we have $\cS^d_0(z_0,\dots,z_{N+1})\subseteq \cN\cN(W,2)$.
\end{lemma}
\begin{proof}
For any function $f=(f_1,\dots,f_d)\in \cS^d_0(z_0,\dots,z_{N+1})$, each component can be written as $f_i = \sum_{n=1}^N c_{i,n}h_n$. For each $f_i$, we can decompose the indices as $\{1,\dots,N\} = \Lambda^i_+ \cup \Lambda^i_-$, where for each $n\in \Lambda^i_+$, we have $c_{i,n}\ge 0$ and for each $n\in \Lambda^i_-$, we have $c_{i,n}< 0$. We then divide each of $\Lambda^i_+$ and $\Lambda^i_-$ into at most $3q$ sets, which are denoted by $\Lambda^i_k$, $1\le k\le 6q$, such that if $n,n'\in \Lambda^i_k$, then the principal breakpoints $z_{jq}, z_{j'q}$ of $h_n, h_{n'}$ respectively, satisfy the separation property $|j-j'|\ge 3$. Then, we can write
$$
f_i = \sum_{k=1}^{6q} f_{i,k}, \qquad f_{i,k} := \sum_{n\in \Lambda^i_k} c_{i,n} h_n, 
$$
where we set $f_{i,k} =0$ if $\Lambda^i_k = \emptyset$. By construction, in the second summation, the $h_n$, $n\in \Lambda^i_k$, have disjoint supports and the $c_{i,n}$ have same sign.

Next, we show that each $f_{i,k}$ is of the form $\pm \sigma(g_{i,k}(z))$, where $g_{i,k}$ is some linear combination of the $\sigma(z-z_{jq})$. First consider the case that the coefficients $c_{i,n}$ in $\Lambda^i_k$ are all positive. Then, we can construct a CPwL function $g_{i,k}$ that takes the value $c_{i,n}$ for the principal breakpoints $z_{jq}$ of $h_n$ with $n\in \Lambda^i_k$ and takes negative values for other principal breakpoints such that it vanishes at the leftmost and rightmost breakpoints of all $h_n$ with $n\in \Lambda^i_k$. This is possible due to the separation property of $\Lambda^i_k$ (an explicit construction strategy can be found in the appendix of \citep{daubechies2021nonlinear}). By this construction, we have $f_{i,k} = \sigma(g_{i,k}(z))$. A similar construction can be applied to the case that all coefficients $c_{i,n}$ in $\Lambda^i_k$ are negative and leads to $f_{i,k} = -\sigma(g_{i,k}(z))$.

Finally, each $f_i = \sum_{k=1}^{6q} f_{i,k}$ can be computed by a network whose first layer has $W$ neurons that compute $\sigma(z-z_{jq})$, $1\le j\le W$, second layer has at most $6q$ neurons that compute $\sigma(g_{i,k}(z))$, and output layer weights are $\pm 1$ or $0$. Since the first layers of these networks are the same, we can stack their second layers and output layers in parallel to produce $f=(f_1,\dots,f_d)$, then the width of the stacked second layer is at most $6dq\le W$. Hence, $f\in \cN\cN(W,2)$.
\end{proof}

We can use Lemma \ref{CPwL two layer} as a building block to represent CPwL functions with large number of breakpoints. Lemma \ref{CPwL} is a consequence of the following lemma by change of variables $W'=W+d+1$ and $L'=2L$.

\begin{lemma}
Suppose that $W\ge 6d$, $L\ge 1$ and $N\le W\lfloor \frac{W}{6d} \rfloor L$. Then for any $z_0<z_1<\dots<z_N<z_{N+1}$, we have $\cS^d(z_0,\dots,z_{N+1}) \subseteq \cN\cN(W+d+1,2L)$.
\end{lemma}

\begin{proof}
By applying a linear transform to the input and adding extra breakpoints if necessary, we can assume that $z_0=0$ and $z_{N+1}=1$, where $N=qWL$ with $q=\lfloor \frac{W}{6d} \rfloor$. For any $f=(f_1,\dots,f_d)\in \cS^d(z_0,\dots,z_{N+1})$, we denote $x_n = (x_{n,1},\dots, x_{n,d}) := f(z_n)$, where $x_{n,i} = f_i(z_n)$. We define 
$$
g_{0,i}(z) := x_{0,i} + (x_{N+1,i}-x_{0,i})(\sigma(z) - \sigma(z-1)),
$$
then $g_{0,i}$ is linear on $(0,1)$ and $g_{0,i}(z)=f_i(z)$ on $(-\infty,0)\cup(1,\infty)$. Let $g_0=(g_{0,1},\dots,g_{0,d})$, then $f-g_0 \in \cS^d_0(z_0,\dots,z_{N+1})$. We can decompose $f-g_0 = \sum_{l=1}^L g_l$, where $g_l\in \cS^d_0(z_0,\dots,z_{N+1})$ is the CPwL function that agree with $f-g_0$ at the points $z_i$ with $i=(l-1)qW+1,\dots,lqW$ and takes the value zero at other breakpoints. Obviously, $g_l \in \cS^d_0(z_{(l-1)qW},\dots,z_{lqW+1})$ and hence $g_l\in \cN\cN(W,2)$ by Lemma \ref{CPwL two layer}.

Next, we construct a network with special architecture of width $W+d+1$ and depth $2L$ that computes $f$. We reserve the first top neuron on each hidden layer to copy the non-negative input $\sigma(z)$. And the last $d$ neurons are used to collect intermediate results and are allowed to be ReLU-free. Since each $g_l \in \cN\cN(W,2)$, we concatenate the $L$ networks that compute $g_l$, $l=1,\dots,L$, and thereby produce $f-g_0$. Observe that $g_0 \in \cN\cN(d+1,1)$, we can use the last $d$ neurons on the first two layers to compute $g_0(z)$. Therefore, $f$ can be produced using this special network. The whole network architecture is showed in Figure \ref{special network architecture}.

\begin{figure}[htb]
\centering
\includegraphics[width=\textwidth]{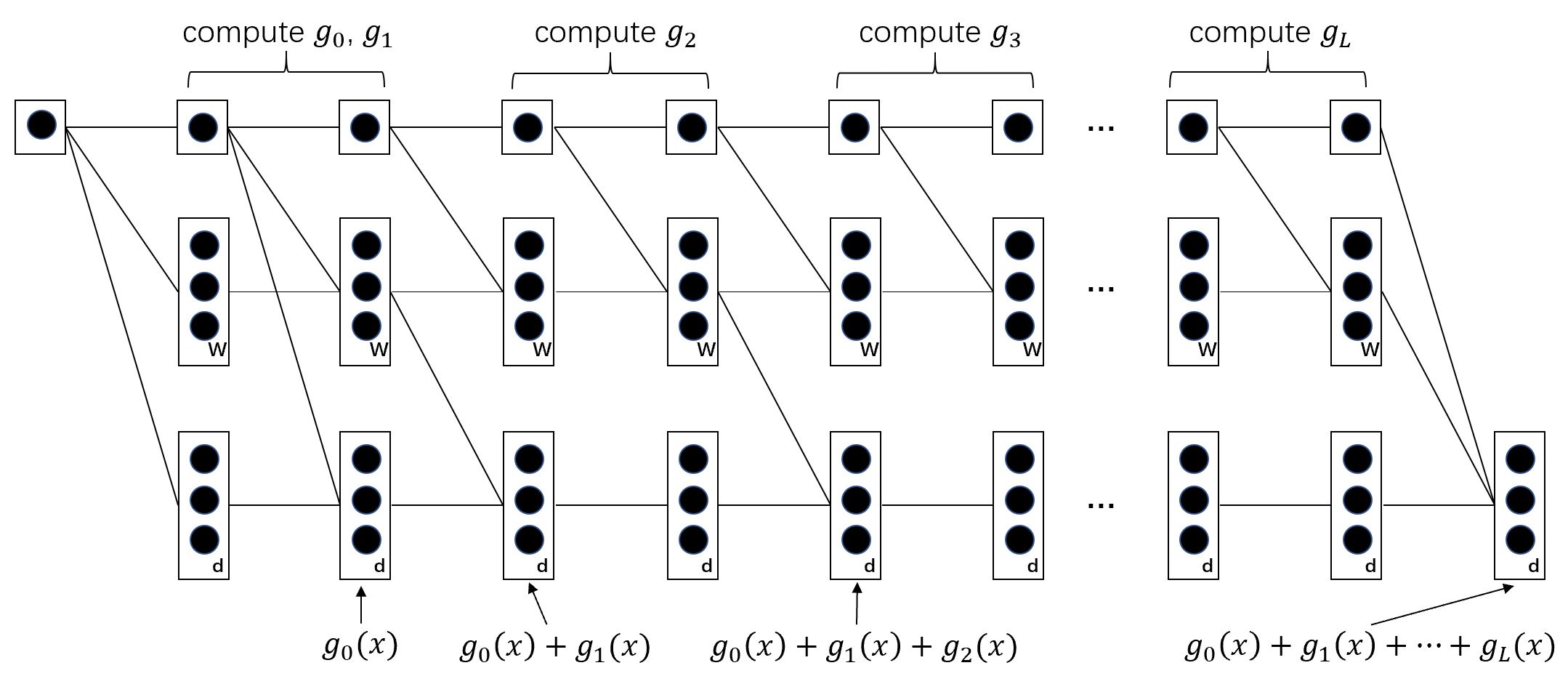}
\caption{The architecture of a neural network that produces $f= \sum_{l=0}^{L}g_l$. The letter on the lower right corner of each rectangle indicates the number of neurons in the rectangle.}
\label{special network architecture}
\end{figure}

Finally, suppose $S_l(z)$ is the output of the last $d$ neurons in layer $l$. Since $S_l(z)$ must be bounded, there exists a constant $C_l$ such that $S_l(z)+C_l>0$ and hence $S_l(z) = \sigma(S_l(z)+C_l)-C_l$. Thus, even though we allow the last $d$ neurons to be ReLU-free, the special network can also be implemented by a ReLU network with the same size. Consequently, $f\in \cN\cN(W+d+1,2L)$, which completes the proof.
\end{proof}

\bibliographystyle{plainnat}
\bibliography{references}

\end{document}